\def \R {\mathbb{R}}

\def \f {\phi}

\def \AA {\forall}
\def \EE {\exists}

\documentclass{SCIS2017}
\def \AA {\forall}
\begin{document}

%%%%%%%%%%%%%%%%%%%%%%%%%%%%%%%%%%%%%%%%%%%%%%%%%%%%%%%
%%% Authors do not modify the information below
%%% ×÷Õß²»ÐèÒªÐ޸Ĵ˴¦ÐÅÏ¢
\ArticleType{RESEARCH PAPER}
%\SpecialTopic{}
\Year{2017}
\Month{January}
\Vol{60}
\No{1}
\DOI{}
\ArtNo{}
\ReceiveDate{}
\AcceptDate{}
\OnlineDate{}
%%%%%%%%%%%%%%%%%%%%%%%%%%%%%%%%%%%%%%%%%%%%%%%%%%%%%%%

%%% title: ±êÌâ
%%%   \title{title}{title for citation}
\title{On Convergence Property of Implicit Self-paced Objective}{On Convergence Property of Implicit Self-paced Objective}

%%% Corresponding author: ͨÐÅ×÷Õß
%%%   \author[number]{Full name}{{email@xxx.com}}
%%% General author: Ò»°ã×÷Õß
%%%   \author[number]{Full name}{}

\author[1]{Zilu Ma}{}
\author[1]{Shiqi Liu}{}
\author[1]{Deyu Meng}{{dymeng@mail.xjtu.edu.cn }}
%\affil{    \normalsize School of Mathematics and Statistics, Xi'an Jiaotong University}

%\author[1]{Aaa AUTHOR}{}
%\author[1,2]{Bbb AUTHOR}{{bauthor@xxx.com}}
%\author[2]{Ccc AUTHOR}{}
%\author[3]{Ddd AUTHOR}{}

%%% Author information for page head. ҳüÖеÄ×÷ÕßÐÅÏ¢
\AuthorMark{Zilu Ma, Shiqi Liu, Deyu Meng}

%%% Authors for citation. Ê×Ò³ÒýÓÃÖеÄ×÷ÕßÐÅÏ¢
\AuthorCitation{Zilu Ma, Shiqi Liu, Deyu Meng, et al}

%%% Authors' contribution. ͬµÈ¹±Ï×
%\contributions{Authors A and B have the same contribution to this work.}

%%% Address. µØÖ·
%%%   \address[number]{Affiliation, City {\rm Postcode}, Country}
%\address[1]{Affiliation, City {\rm 000000}, Country}
%\address[2]{Affiliation, City {\rm 000000}, Country}
%\address[3]{Affiliation, City {\rm 000000}, Country}
\address[1]{Institute for Information and System Sciences and Ministry of Education Key Lab of Intelligent Networks \\and Network Security, Xi'an Jiaotong University, Xi'an {\rm 710049}, China}

%%% Abstract. ÕªÒª
\abstract{ Self-paced learning (SPL) is a new methodology that simulates the learning principle of humans/animals to start learning easier aspects of a learning task, and then gradually take more complex examples into training. This new-coming learning regime has been empirically substantiated to be effective in various computer vision and pattern recognition tasks. Recently, it has been proved that the SPL regime has a close relationship to a implicit self-paced objective function. While this implicit objective could provide helpful interpretations to the effectiveness, especially the robustness, insights under the SPL paradigms, there are still no theoretical results strictly proved to verify such relationship. To this issue, in this paper, we provide some convergence results on this implicit objective of SPL. Specifically, we prove that the learning process of SPL always converges to critical points of this implicit objective under some mild conditions. This result verifies the intrinsic relationship between SPL and this implicit objective, and makes the previous robustness analysis on SPL complete and theoretically rational.}

%%% Keywords. ¹Ø¼ü´Ê

\maketitle

%%%%%%%%%%%%%%%%%%%%%%%%%%%%%%%%%%%%%%%%%%%%%%%%%%%%%%%
%%% The main text. ÕýÎIJ¿·Ö
%%%%%%%%%%%%%%%%%%%%%%%%%%%%%%%%%%%%%%%%%%%%%%%%%%%%%%%
\section{Introduction}

%%changed by 647 \3.26
%Self-paced learning (SPL) is a recently raised methodology designed through simulating the learning principle of humans/animals~\cite{kumar2010self}. A variety of SPL realization schemes have been designed, and empirically substantiated to be effective in different computer vision and pattern recognition tasks, such as object detector adaptation \cite{SPShift}, specific-class segmentation learning \cite{SPSegmentation}, visual category discovery \cite{SPVCD}, concept learning \cite{WELL}, long-term tracking \cite{SPTrack}, and multimedia event detection~\cite{MED14}.

Self-paced learning (SPL) is a recently raised methodology designed through simulating the learning principle of humans/animals~\cite{kumar2010self}. A variety of SPL realization schemes have been designed, and empirically substantiated to be effective in different computer vision and pattern recognition tasks, such as object detector adaptation \cite{SPShift}, specific-class segmentation learning \cite{SPSegmentation}, visual category discovery \cite{SPVCD}, concept learning \cite{WELL}, long-term tracking \cite{SPTrack}, graph trimming\cite{Yue2016Semi}, co-saliency detection\cite{Zhang2015A}, matrix factorization\cite{Zhao2015Self}, face identification\cite{Lin2017Active}, and multimedia event detection~\cite{MED14}.

To explain the underlying effectiveness mechanism inside SPL, \cite{2015arXiv151106049M} firstly provided some new theoretical understandings under the SPL scheme.
%% changed by 647 /3.25
%Specifically, this work proved that the solving strategy on SPL accords with a majorization minimization (MM) algorithm implemented on an implicit objective function.
Specifically, this work proved that the alternative optimization strategy (AOS) on SPL accords with a majorization minimization (MM) algorithm implemented on an implicit objective function.
Furthermore, it is found that the loss function contained in this implicit objective has a similar configuration with non-convex regularized penalty (NCRP), leading to a rational interpretation to the robustness insight under SPL.

However, such understanding is still not theoretically strict. The theory in \cite{2015arXiv151106049M} can only guarantee that during the iterations of SPL solving process (i.e., the MM algorithm), the implicit objective is monotonically decreasing, while cannot prove any convergence results on this implicit objective theoretically. However, this theoretical result regarding this implicit objective is critical to the soundness of the robustness insight explanation of SPL, which guarantees to settle the convergence point of the algorithm down on the expected implicit objective, and intrinsically relate the original SPL model and this implicit objective.

To this theoretical issue of SPL, in this paper, we prove that the optimization of the implicit objective actually converges to critical points of original SPL problem under satisfactorily weak conditions. This result provides an affirmative answer to our guess that the SPL intrinsically optimizes a robust implicit objective.

In what follows, we will first introduce some related background of this research, and then we provide the main theoretical result of this work.

\section{Related work}
In this section, we first briefly introduce the definition of SPL, and then provides its relationship to the implicit objective of NCRP.

\subsection{The SPL objective}
Given training data set $\{(x_i,y_i)\}_{i=1}^N$,
many machine learning problems need to minimizing the following form of objective function:
\begin{equation*}
	J(w) = \phi_\lambda(w) + \sum_{i=1}^N L(y_i, g(x_i,w)),
\end{equation*}
where $w\in \R^D$ is variables to be solved, $\phi_\lambda$ is a regularizer parameter, $L$ is the loss function and $g(\cdot,w)$ is the parametrized learning machine, like a discriminative or a regression function.

To improve the robustness, specially avoiding the negative influence brought by large-noise-outliers, SPL imposes additional importance weights $v=(v_1,\cdots,v_n)$ to loss functions of all samples, adjusted by a self-paced regularizer (SP-regularizer). Here,
each $v_i\in [0,1]$ represents how much extent the sample $(x_i,y_i)$ will be trained in the learning process.
The \textbf{self-paced objective} can then be designed as \cite{SPaR}:
\begin{equation}
	E(w,v;\lambda)= \phi_\lambda(w) + \sum_{i=1}^N v_i L(y_i, g(x_i,w)) + f_\lambda(v_i),\label{eq2}
\end{equation}
where $f$ is the \textbf{SP-regularizer}, satisfying the following conditions:
%\begin{enumeratebrac}
% 	\item $v\mapsto f_\lambda(v)$ is convex on $[0,1]$;
% 	\item Let
% 	\begin{equation} \nonumber
% 		v_\lambda^*(l) = \arg \min_{v\in [0,1]} \left\{ vl+f_\lambda(v) \right\},
% 	\end{equation}
% 	then $l\mapsto v_\lambda^*(l)$ is non-increasing, and
% 	\begin{equation} \nonumber
% 		\lim_{l\to 0} v_\lambda^*(l)=1,\quad \lim_{l\to \infty} v_\lambda^*(l)=0;
% 	\end{equation}
% 	\item $\lambda \mapsto v_\lambda^*(l)$ is non-decreasing, and
% 	\begin{equation} \nonumber
% 		\lim_{\lambda\to 0} v_\lambda^*(l)=0,\quad \lim_{\lambda\to \infty} v_\lambda^*(l)\leq 1.
% 	\end{equation}
% \end{enumeratebrac}
\begin{enumerate}
  \item $v\mapsto f_\lambda(v)$ is convex on $[0,1]$;
  \item Let
 	\begin{equation} \nonumber
 		v_\lambda^*(l) = \arg \min_{v\in [0,1]} \left\{ vl+f_\lambda(v) \right\},
 	\end{equation}
 	then $l\mapsto v_\lambda^*(l)$ is non-increasing, and
 	\begin{equation} \nonumber
 		\lim_{l\to 0} v_\lambda^*(l)=1,\quad \lim_{l\to \infty} v_\lambda^*(l)=0;
 	\end{equation}
  \item $\lambda \mapsto v_\lambda^*(l)$ is non-decreasing, and
 	\begin{equation} \nonumber
 		\lim_{\lambda\to 0} v_\lambda^*(l)=0,\quad \lim_{\lambda\to \infty} v_\lambda^*(l)\leq 1.
 	\end{equation}
\end{enumerate}
%% ==== mazilu 3.26 ======
Throughout this paper, we shall assume that $v_\lambda^*(l)$ can be uniquely determined and thus can be seen as
a real-valued function instead of a set-valued function.
%% ==== mazilu 3.26 ======

The three conditions in the definition above provide basic principles for constructing a SP-regularizer. Condition 2 indicates that the model inclines to select easy samples (with smaller losses) in favor of complex samples (with larger losses). Condition 3 states that when the model ``pace" (controlled by the pace parameter $\lambda$) gets larger, it tends to incorporate more, probably complex, samples to train a ``mature" model. The convexity in Condition 1 further ensures the soundness of this regularizer for optimization.

%% 647 adds an example /3.25
The existence of the SP-regularizer can be illustrated by the following example.

Let the SP-regularizer be\begin{equation} \nonumber f_\lambda(v)=\lambda v(\log v-1),
\end{equation} then it yields
\begin{equation} \nonumber
v_\lambda^*(l)=e^{-\lambda^{-1}l}.
\end{equation}
It is easy to verify that $v_\lambda^*(l)$ satisfies the above conditions.

In the following, we shall write:
\begin{equation} \nonumber
	l_i(w) = L(y_i, g(x_i,w)), \quad i=1,\cdots,N
\end{equation}
%% ==== mazilu 3.26 ======
%for notation convenience.
for simplicity.
%% ==== mazilu 3.26 ======

\subsection{The implicit NCRP objective}
Let
\begin{equation} \nonumber
	F_\lambda(l)=\int_0^l v_\lambda^*(\tau) d\tau.
\end{equation}
Since $v_\lambda^*$ is non-increasing, the set of its discontinuous points is countable and consists only of jump discontinuity.
Thus $v_\lambda^*$ is integrable and $F_\lambda$ is absolutely continuous and concave.
We now define
\begin{equation}
	G_\lambda(w) = \f_\lambda(w) + \sum_{i=1}^{N} (F_\lambda\circ l_i)(w) \label{eq1}
\end{equation}
as the \textbf{implicit objective}, where $g\circ f$ denotes that $g$ composed with $f$.
An interesting observation is that this implicit SPL objective has a close relationship to NCRP widely investigated in machine learning and statistics, which provides some helpful explanation to the robustness insight under SPL \cite{2015arXiv151106049M}.

The original utilized AOS algorithm for solving the SPL problem is designed by
performing coordinate descent calculation on $E(w,v;\lambda)$, i.e., iterating through the process as:
\begin{equation} \nonumber
	(w^{k-1},v^{k-1})\to (w^{k-1},v^k)\to (w^k,v^k).
\end{equation}
Specifically, given $(w^0,v^0)$, if we have finished $(k-1)$ steps, then the AOS algorithm need to iteratively calculating the following two subproblems:
%% changed by 647 /3.26
%\begin{equation} \nonumber
%	v^{k} \in \arg \min_v E(w^{k-1},v;\lambda) = \arg \min_v \left\{  \sum_{i=1}^N v_i l_i(w) + f_\lambda(v_i) \right\} ,
%\end{equation}
%\begin{equation} \nonumber
%	w^k \in \arg \min_w E(w,v^k;\lambda) = \arg \min_w \left\{  \phi_\lambda(w) + \sum_{i=1}^N v_i l_i(w) \right\}.
%\end{equation}
%\begin{equation} \nonumber
%	v^{k} \in \arg \min_v E(w^{k-1},v;\lambda) = \arg \min_v \left\{  \sum_{i=1}^N v_i l_i(w^{k-1}) + f_\lambda(v_i) \right\} ,
%\end{equation}
%\begin{equation} \nonumber
%	w^k \in \arg \min_w E(w,v^k;\lambda) = \arg \min_w \left\{  \phi_\lambda(w) + \sum_{i=1}^N v_i^k l_i(w) \right\}.
%\end{equation}
%%
%% ===== mazilu 3.26 ======

% \[
% 	v^{k} \in \argmin_v E(w^{k-1},v;\l) = \argmin_v \left\{  \sum_{i=1}^N v_i l_i(w^{k-1}) + f_\l(v_i) \right\} ,
% \]
% \[
% 	w^k \in \argmin_w E(w,v^k;\l) = \argmin_w \left\{  \phi_\l(w) + \sum_{i=1}^N v_i^k l_i(w) \right\}.
% \]
\begin{equation} \nonumber
	v^{k} = \arg\min_v E(w^{k-1},v;\lambda) = \arg\min_v \left\{  \sum_{i=1}^N v_i l_i(w^{k-1}) + f_\lambda(v_i) \right\} ,
\end{equation}
\begin{equation} \nonumber
	w^k \in \arg\min_w E(w,v^k;\lambda) = \arg\min_w \left\{  \phi_\lambda(w) + \sum_{i=1}^N v_i^k l_i(w) \right\}.
\end{equation}
Note that the first subproblem is feasible since we have assumed that $v^*_\lambda$ can be uniquely determined.
Indeed, using the notation of $v_\lambda^*$, we have
\begin{equation} \nonumber
	v_i^k = v_\lambda^*(l_i(w^{k-1})),\quad i=1,\cdots,N.
\end{equation}
%% ===== mazilu 3.26 =======

We then set
\begin{equation} \nonumber
	Q(w|w^*) = \sum_{i=1}^N (F_\lambda\circ l_i)(w^*) + (v_\lambda^*\circ l_i)(w^*)[ l_i(w)-l_i(w^*) ].
\end{equation}
It is easy to deduce that $Q(w|w^*)$ is actually the first-order Taylor series of $F_\lambda$ at $l_i(w^*)$.
Based on the concavity of $F_\lambda$, we know that
\begin{equation} \nonumber
	U(w|w^*)=\f_\lambda(w) + Q(w|w^*)
\end{equation}
%% changed by 647 /3.25
%constitutes a lower bound of $G_\lambda(w)$ (as defined in Eq. (\ref{eq1})), which provides a qualified surrogate function for MM algorithm.
constitutes a upper bound of $G_\lambda(w)$ (as defined in Eq.~(\ref{eq1})), which provides a qualified surrogate function for MM algorithm.
%%

%% changed by 647 /3.25
%Let $\left\{ w^k,v^k \right\}$ be the parameter sequence produced by the coordinate descent algorithm of the original objective $E(w,v;\lambda)$ and $\left\{ \tilde{w}^k \right\}$ be produced by performing MM algorithm on $G_\lambda$ by taking $U(w|w^*)$ as the surrogate function.

%% ===== mazilu 3.26 =====

% Let $\left\{ w^k,v^k \right\}$ be the parameter sequence produced by the
% AOS on the original objective $E(w,v;\l)$ and
% $\left\{ \tilde{w}^k \right\}$ be produced by performing MM algorithm on $G_\l$ by taking $U(w|w^*)$ as the surrogate function.
% %%
% If we assume that $w^k=\tilde{w}^k$ and that every minima can be obtained and is unique,
% then when it comes to the $(k+1)^{\text{th}}$ step, we can prove that
% \begin{eqnarray*}
% 	\tilde{w}^{k+1}&=&\argmin_w \f_\l(w) + Q(w|\tilde{w}^k) \\
% 	&=& \argmin_w \f_\l(w) + Q(w|w^k) \\
% 	&=& \argmin_w \f_\l(w) + \sum_{i=1}^N (v_\l^*\circ l_i)(w^k) \cdot l_i(w) \\
% 	&=& \argmin_w \f_\l(w) + \sum_{i=1}^N v^{k+1} l_i(w) \\
% 	&=& \argmin_w E(w,v^{k+1};\l) = w^{k+1}.
% \end{eqnarray*}
% Thus, these two optimization algorithms (AOS/MM) conducting on the two different objective functions ($E(w,v;\l)$/$G_\l(w)$) are intrinsically equivalent.

One of the key issues in \cite{2015arXiv151106049M} is that
if  $\{w^k\}$ is produced by AOS algorithm of $E(w,v;\lambda)$, then
it can also be produced by performing MM algorithm on $G_\lambda$ and vice versa.
We prove one side by induction. The other side is totally the same.
Suppose we have proved that $w^k$ can be produced by performing MM algorithm on $G_\lambda$ at $k^{\text{th}}$ step.
When it comes to the $(k+1)^{\text{th}}$ step,
\begin{eqnarray}
	w^{k+1} &\in& \arg\min_w E(w,v^{k+1};\lambda)\nonumber \\
	&=& \arg\min_w \phi_\lambda(w) + \sum_{i=1}^N v^{k+1}_i l_i(w) \nonumber\\
	&=& \arg\min_w \phi_\lambda(w) + \sum_{i=1}^N v_\lambda^*( l_i(w^k)) \cdot l_i(w) \nonumber\\
	&=&\arg\min_w \phi_\lambda(w) + Q(w|w^k) = \arg\min_w U(w|w^k). \label{eq3}
\end{eqnarray}
Thus we have proved our claim that
these two optimization algorithms (AOS/MM) conducting on the two different objective functions ($E(w,v;\lambda)$/$G_\lambda(w)$) are intrinsically equivalent.

We then need to prove whether every convergence point of MM algorithm, or equivalently, that of the AOS algorithm on the SPL objective, is at least a critical point of $G_\lambda$.

%% ===== mazilu 3.26 =====
\section{The main convergence result}

Actually, the proof of the convergence of MM algorithm is basically the same as that of the EM algorithm (see \cite{wu1983convergence}) only with some obvious changes, as discussed in \cite{vaida2005parameter}. And the convergence of EM and MM is indeed a corollary of a global convergence theorem of Zangwill (see \cite{zangwill1969nonlinear}). We can generalize the proof to the case of variational analysis. Before that, we need to clarify some terminologies which can be referred to in \cite{rockafellar2009variational}.

A function $f:\R^d \to \overline{\R}$ is said to be \textbf{lower semi-continuous} or simply \textbf{lsc}
if
\begin{equation} \nonumber
	\text{lev}_{f\leq \alpha} := \{x: f(x) \leq \alpha \}
\end{equation}
is closed for any $\alpha\in \R$.
$f$ is said to be \textbf{level-bounded} if $\text{lev}_{f\leq \alpha}$ is bounded for any $\alpha$.
And $f$ is called \textbf{coercive} if $\lim_{|x|\to \infty}f(x) = \infty$. Note
that coercive functions are level-bounded.
A \textbf{critical point} $x$ of $f$ means that $0\in \partial f(x)$, where $\partial$ stands for the \textbf{subdifferential}\cite{rockafellar2009variational}.

The main theorem of this paper can then be stated as follows.

\begin{theorem} \label{thm:mm}
	Suppose that the objective function of MM algorithm, $G:\mathbb{R}^D\to \overline{\mathbb{R}}$,
	is lsc and level-bounded,
	and that the surrogate function at $w^*$ is $U(\cdot|w^*)$, which is
	lsc as a function on $\R^{2D}$, and satisfies
	\begin{equation} \nonumber
		\partial U(w|w) \subset \partial G(w)  ,\quad \AA w\in \R^D,
	\end{equation}
	where $\partial U(w|w^*)$ is the partial subdifferential in $w$. Then
	for any initial parameter $w^0$, every cluster point of
	the produced sequence $\{w^k\}$ of MM algorithm is a critical point of $G$.	
\end{theorem}
\begin{proof}
See the appendix.
\end{proof}

For our problem, we can give a sufficient condition of convergence, which is easy to verify and satisfied by most of the current SPL variations.

\begin{theorem}\label{thm:conv1}
In the SPL objective as defined in Section 2.1, suppose $L$ is bounded below, $w\mapsto L(y, g(x,w))$ is continuously differentiable,
$v_\lambda^*(\cdot)$ is continuous, and $\f_\lambda$ is coercive and lsc.
Then for any initial parameter $w^0$, every cluster point of
	the produced sequence $\{w^k\}$, obtained by the AOS algorithm on solving Eq.~(\ref{eq2}), is a critical point of the implicit objective $G_\lambda$ as defined in Eq.~(\ref{eq1}).
\end{theorem}
\begin{proof}
It is obvious that $G_\lambda$ is lsc and level-bounded and
$U$ is lsc as a function on $\R^{2D}$ with these assumptions.
And the continuity of $v_\lambda^*$ makes $F_\lambda$ continuously differentiable. Then we have
\begin{eqnarray*}
	\partial G_\lambda(w^*)& = &
	\partial \f_\lambda(w^*) +
		\sum_{i=1}^N F_\lambda'(l_i(w^*)) \nabla l_i(w^*) \\
	&=&\partial \f_\lambda(w^*) +
		\sum_{i=1}^N (v_\lambda^*\circ l_i)(w^*) \nabla l_i(w^*)\\
	&=&\partial U(w^*|w^*).
\end{eqnarray*}
Based on Theorem 1, for any initial parameter $w^0$, every cluster point of
the produced sequence $\{w^k\}$ is a critical point of $G_\lambda$.
The proof is then completed.
\end{proof}

From the theorem, we can see that the AOS algorithm generally used to solving the SPL problem can be guaranteed to convergent to a critical point of the implicit NCRP objective $G_\lambda$. The intrinsic relationship between two objectives can then be constructed.

Note that in the above theorem, it is required that every minimization step in MM algorithm exactly attains the minima of the surrogate function $U(w|w^k)$, i.e.,
\begin{equation}
	U(w^{k+1}|w^k) = \min U(\cdot|w^k). \label{eq4}
\end{equation}
This is generally hard to achieve in real applications, especially for those learning models without closed-form solution. We thus want to further relax the condition to allow a relatively weaker solution ``with errors" in implementing the MM algorithm on the surrogate function. That is, we can weaken the condition (\ref{eq4}) as:
\begin{equation} \nonumber
	U(w^{k+1}|w^k) \leq \min U(\cdot|w^k) + \epsilon_k,
\end{equation}
where $\epsilon_1, \epsilon_2, \cdots$ is a non-negative sequence satisfying $\{\epsilon_k\}\in l^1$, i.e., $\sum_k \epsilon_k<\infty$.

Under this relaxed condition, we can still prove the convergence result of SPL in the following algorithm

\begin{theorem}
In the SPL objective as defined in Section 2.1, suppose $L$ is bounded below, $w\mapsto L(y, g(x,w))$ is continuously differentiable,
$v_\lambda^*(\cdot)$ is continuous, and $\f_\lambda$ is coercive and lsc.
Let $w^0$ be an arbitrary initial parameter, and $\{w^k\}$ be
the sequence obtained by the AOS algorithm on solving Eq.~(\ref{eq2}) with errors $\{\epsilon_k\ge 0 \}\in l^1$, that is,
\[
    E(w^k,v^k;\lambda) \leq \min E(\cdot,v^k;\lambda) + \epsilon_k, \quad \forall k\geq 1.
\]
Then every cluster point of $\{w^k\}$ is a critical point of the implicit objective $G_\lambda$ as defined in Eq.~(\ref{eq1}).
\end{theorem}

Based on the theorem, we can then confirm the intrinsic relationship between SPL and its implicit objective.

\section{Conclusion}
In this paper, we have proved that the learning process of traditional SPL regime can be guaranteed to converge to rational critical points of the corresponding implicit NCRP objective. This theory helps confirm the intrinsic relationship between SPL and this implicit objective, and thus verifies previous robustness analysis of SPL on the basis of the understanding of such relationship. Besides, we have used some new theoretical skills for the proof of convergence, which inclines to be beneficial to the previous MM and EM convergence theories to a certain extent.

\appendix

\section{Proof of Theorem~\ref{thm:mm}}
Theorem 1 is actually a corollary of a stronger version of Zangwill's global convergence theorem
\cite[page 91]{zangwill1969nonlinear}. We first need to give the following lemmas.
\begin{lemma}
	If $f$ is lsc, $x_n\to x$, and $\{f(x_n)\}$ is non-increasing, then $f(x_n)\to f(x)$.
\end{lemma}
\begin{proof}
	\begin{equation} \nonumber
		f(x) = \liminf_{n\to \infty} f(x_n) = \lim_{k\to \infty} \inf_{n \geq k} f(x_n)
		= \inf_{n\geq 1} f(x_n) = \lim_{n\to \infty} f(x_n).
	\end{equation}
\end{proof}
\bigskip

\begin{lemma} \label{lem:Zang}
Suppose that $X$ is a finite-dimensional Euclidean space, $M$ is a set-valued mapping from $X$ to $X$
and that
$\{x_k\}$ is produced by $M$, which means
\begin{equation} \nonumber
	x_{k+1}\in M(x_k),\quad \AA k\geq 0.
\end{equation}
$\Gamma$ is a subset of $X$ that we are interested at, called the "solution set" and satisfying
\begin{enumerate}
	\item There is a compact subset $K$, such that $x_k\in K,\AA k$,
	\item $M$ is outer semicontinuous on $X\setminus \Gamma$, that is
	\begin{equation} \nonumber
		x_k\to x\ \text{in}\ X\setminus \Gamma \implies M(x_k)\to M(x).
	\end{equation}
	\item There is a lsc function $G$ defined on $X$, such that
		\begin{enumerate}
			\item $G(y)<G(x),\AA y\in M(x),x\notin \Gamma$,
			\item $G(y)\leq G(x),\AA y\in M(x),x\in \Gamma$,		
		\end{enumerate}
\end{enumerate}
then all the cluster points of $\{x_k\}$ are in $\Gamma$, and $\EE \bar{x}\in \Gamma$,
such that $G(x_k)$ is non-increasing and convergent to $G(\bar{x})$.
\end{lemma}
\textbf{Note:} we will repeatedly use the fact that $\{G(x_k)\}$ is non-increasing.
Without loss of generality, we can assume that $n_1<n_2<\cdots$ when we take a subsequence
$\{x_{n_k}\}$ of $\{x_{n}\}$.

\begin{proof}
	(1) Suppose $x^*$ is a cluster point of $\{x_k\}$.
	The existence of $x^*$ is guaranteed by the compactness of $K$.
	Thus there exists a subsequence $\{x_{n_k}\}$, such that
	$x_{n_k}\to x^*,(k\to \infty)$.
	Since $\{G(x_k)\}$ is non-increasing, based on Lemma 3, it holds that
	\begin{equation} \nonumber
		G(x^*) = \lim_{k\to \infty} G (x_{n_k}).
	\end{equation}

	Denote $G^* = G(x^*)$, and then we prove that $G(x_n)\to G^*(n\to \infty)$.
	This is because $\AA \epsilon>0, \EE k_0>0,$
	such that
	\begin{equation} \nonumber
		G(x_{n_k}) - G^* < \epsilon, \quad \AA k\geq k_0.
	\end{equation}
	When $n\geq n_{k_0}$,
	\begin{equation} \nonumber
		G(x_n) - G^* = G(x_n)-G(x_{n_{k_0}}) + G(x_{n_{k_0}}) - G^* < 0+\epsilon=\epsilon.
	\end{equation}
	There exists $k_1>0$, such that $n<n_{k_1}$, and thus
	\begin{equation} \nonumber
		G(x_n) - G^* = G(x_n)-G(x_{n_{k_1}}) + G(x_{n_{k_1}}) - G^* \geq 0+0=0.
	\end{equation}
	Therefore,
	\begin{equation} \nonumber
		0 \leq G(x_n) - G^* < \epsilon, \quad \AA n \geq n_{k_0},
	\end{equation}
	which indicates $G(x_n)\to G^*$.

	(2) If $x^*\notin \Gamma$, take a subsequence
		\begin{equation} \nonumber
			y_k = x_{n_k+1} \in M(x_{n_k}).
		\end{equation}
		Since $y_k$ all lie in $K$, there exists a
		subsequence $\{y_{k_l}\}$, such that $y_{k_l}\to \bar{x},(l\to \infty)$.
		Since $M$ is outer semicontinuous,
		$\bar{x}\in M(x^*)$.
		Based on Lemma 3, we know that $G(y_{k_l})\to G(\bar{x}),(l\to \infty)$.
		Due to the properties of $G$,
		\begin{equation} \nonumber
			G(\bar{x}) < G(x^*) = \lim_{n\to \infty} G(x_n) = \lim_{l\to \infty} G(y_{k_l})
			= G(\bar{x}),
		\end{equation}
		a contradiction.

\end{proof}
\bigskip

We then provide a proof of Theorem~\ref{thm:mm}.

\begin{proof}[Proof of Theorem~\ref{thm:mm}]\ \

	Let $\Gamma$ be the set of critical points of $G$, and
	\begin{equation} \nonumber
		M(w^*) = \arg \min_{w} U(w|w^*).
	\end{equation}
	By the descending property of MM algorithm, $G(w)\leq G(w^*),\AA w\in M(w^*)$.
	Condition 3b is satisfied.

	Condition 1: since $G$ is lsc and level-bounded,
	$K=\text{lev}_{G\leq w^0}$ is closed and bounded, and thus compact.
	By the descending property of MM algorithm, all the parameters $w^k$ lie in $K$.

	Condition 2: suppose $w^k\to w^*,v^k\to v^*,v^k\in M(w^k)$, and then $\AA w\in \R^D$, it holds that
	\begin{equation} \nonumber
		U(v^k|w^k) \leq U(w|w^k).
	\end{equation}
	Taking infimal limit on both sides when $k\to \infty$, we have
	\begin{equation} \nonumber
		U(v^*|w^*) = \liminf_{k\to \infty} U(v^k|w^k) \leq \liminf_{k\to \infty} U(w|w^k)
		= U(w|w^*).
	\end{equation}
	Thus $v^*\in M(w^*)$, which means $M$ is outer semi-continuous.

	Condition 3a: If $w^*\notin \Gamma$, then
	\begin{equation} \nonumber
		0\notin \partial G(w^*) \supset \partial U(w^*|w^*).
	\end{equation}
	By the generalized Fermat theorem (see \cite[10.1]{rockafellar2009variational}),
	$w^*$ is not a minima of $U(\cdot|w^*)$, i.e., $w^*\notin M(w^*)$.
	Since $\AA w\in M(w^*)$,	\begin{equation} \nonumber
		G(w) \leq U(w|w^*) < U(w^*|w^*) = G(w^*).
	\end{equation}
	All the conditions of the proceeding theorem are satisfied.
The proof is then completed.
\end{proof}

\section{Proof of Theorem~3}

Similar to \cite[5.41]{rockafellar2009variational}, we give the following definition:
\begin{definition}
A sequence of set-valued mappings $M_k$ \textbf{converges outer semicontinuously}
to another set-valued mapping $M$, if
\begin{equation} \nonumber
	\limsup_{k\to \infty} M_k(x_k) \subset M(\bar{x}),\quad \AA x_k\to \bar{x},
\end{equation}
that is,
\begin{equation} \nonumber
	x_k\to \bar{x},v_k\in M_k(x_k), v_k\to \bar{v} \implies \bar{v} \in M(\bar{x}).
\end{equation}
\end{definition}
\bigskip

Before giving proof of Theorem 3, we need to prove the following two lemmas.

\begin{lemma}  \label{thm:zangpr}
Let $X$ be an Euclidean space with finite dimension, and $M,M_k,k=1,2,\cdots$ be
set-valued mappings from $X$ to itself.
Suppose that $M_k$ converges outer semicontinuously to $M$, and that
$\{x_k\}$ is produced by $\{M_k\}$, which means
\begin{equation} \nonumber
	x_{k+1}\in M_k(x_k),\quad \AA k.
\end{equation}
Let $\Gamma$ be an arbitrary set, called the "solution set", satisfying
\begin{enumerate}
	\item There is a compact set $K$ such that $x_k\in K,\AA k$,
	\item There is a lsc $\alpha$ defined on $X$, such that
		\begin{enumerate}
			\item $\alpha(y)<\alpha(x),\AA y\in M(x),x\notin \Gamma$;
			\item There is a sequence of non-negative numbers
			$\{\epsilon_k\}\in l^1$, that is $\sum_k \epsilon_k<\infty$, and
			\begin{equation} \nonumber
				\alpha(y_{k+1}) \leq \alpha(x)+\epsilon_k,\quad \AA y_{k+1}\in M_k(x),\AA x, \AA k.
			\end{equation}
		\end{enumerate}
\end{enumerate}
Then all the cluster points of $\{x_k\}$ lie in $\Gamma$, and
$\EE \bar{x}\in \Gamma$, such that $\alpha(x_k)$ converges to $\alpha(\bar{x})$.
\end{lemma}

\begin{proof}
(1) Set $r_k=\sum_{j\geq k} \epsilon_j$, then $r_k\to 0$, and
\begin{equation} \nonumber
	\alpha(x_{k+1})+r_{k+1} \leq \alpha(x_k)+\epsilon_k+r_{k+1} = \alpha(x_k) + r_k.
\end{equation}
Thus $\{\alpha(x_k)+r_k\}$ is non-increasing.\\

(2) Let $x^*$ be a cluster point of $\{x_k\}$, and then
there exists a subsequence $\{x_{n_k}\}$, such that $x_{n_k}\to x^*$.
Since $\alpha$ is lsc, we have
\begin{equation} \nonumber
	\alpha(x^*)=\liminf_{k} \alpha(x_{n_k}) = \liminf_{k} \left( \alpha(x_{n_k}) + r_{n_k} \right)
	= \lim_{k} \left( \alpha(x_{n_k}) + r_{n_k} \right) = \lim_{k} \alpha(x_{n_k}).
\end{equation}
The second equality holds because
\begin{equation} \nonumber
	\liminf_{k} \alpha(x_{n_k})\leq
	\liminf_{k} \left( \alpha(x_{n_k}) + r_{n_k} \right)
	\leq \liminf_{k} \alpha(x_{n_k}) + \limsup_{k} r_{n_k} = \liminf_{k} \alpha(x_{n_k}).
\end{equation}
And we can prove in the same way as in (1) of Lemma \ref{lem:Zang} that
\begin{equation} \nonumber
	\lim_{n} \alpha(x_n) = \alpha(x^*).
\end{equation}
(3) We need to show that $x^*\in \Gamma$. Suppose not, take
\begin{equation} \nonumber
	y_k = x_{n_k+1}\in M_{n_k}(x_{n_k}),
\end{equation}
Due to the compactness of $K$, there is a subsequence $\{y_{k_l}\}$ of $\{y_k\}$,
such that $\EE \bar{x}\in K$, $y_{k_l}\to \bar{x}$.
We can argue in the same way as in (2) to show that $\alpha(y_{k_l})\to \alpha(\bar{x})$.
Since
\begin{equation} \nonumber
	y_{k_l}= x_{n_{k_l}+1}\in M_{n_{k_l}}(x_{n_{k_l}}),
\end{equation}
and $M_k$ converges outer semicontinuously to $M$, we have $\bar{x}\in M(x^*)$.
Thus
\begin{equation} \nonumber
	\alpha(\bar{x})<\alpha(x^*) = \lim_{n} \alpha(x_n) = \lim_{l} \alpha(y_{k_l})= \alpha(\bar{x}),
\end{equation}
a contradiction.

The proof is then completed.
\end{proof}
\bigskip

Now we can prove another lemma using above theoretical result.

\begin{lemma} \label{thm:mmerr}
	Let $F:\mathbb{R}^D\to \overline{\mathbb{R}}$ be the objective of MM algorithm.
	Suppose that $F$ is lsc and level-bounded,
	and that the surrogate function at $w^*$ is $U(\cdot|w^*)$.
	In addition, suppose
	$U(\cdot|\cdot)$ is lsc as a function defined on $\R^{2D}$ whose subgradient satisfies
	\begin{equation} \nonumber
		\partial U(w|w) \subset \partial F(w)  ,\quad \AA w\in \R^D,
	\end{equation}
	where $\partial U(w|w^*)$ is the partial subdifferential with respect to $w$.
	Then for any initial parameter $w^0$,
	all the cluster points of the sequence $\{w^k\}$ produced by MM algorithm "with errors"
	are still critical points of $F$.
	
\end{lemma}

\begin{proof}
	We prove by a direct application of Lemma \ref{thm:zangpr}.
	Let $\Gamma$ be the set consisting of all the critical points of $F$, $\alpha=F$,
	and
	\begin{equation} \nonumber
		M_k(w^*) = \{ w : U(w|w^*) \leq \min U(\cdot|w^*) + \epsilon_k \}.
	\end{equation}
	$M$ is the same as before:
	\begin{equation} \nonumber
		M(w^*) = \arg \min_{w} U(w|w^*).
	\end{equation}
	We first need to show that $M_k$ converges outer semicontinuously to $M$.
	Suppose $w^k\to \bar{w},v^k\in M_k(w^k),v^k\to \bar{v}$, and then
	$\AA w$,
	\begin{equation} \nonumber
		U(v^k|w^k)\leq \min U(\cdot|w^k) + \epsilon_k \leq U(w|w^k) + \epsilon_k.
	\end{equation}
	Taking infimal limit on both sides when $k\to \infty$, we have
	\begin{equation} \nonumber
		U(\bar{v}|\bar{w}) = \liminf_{k}U(v^k|w^k) \leq \liminf_{k}\left( U(w|w^k) + \epsilon_k \right)
		\leq \liminf_{k} U(w|w^k) + \limsup_{k} \epsilon_k = U(w|\bar{w}),
	\end{equation}
	which means $\bar{v}\in M(\bar{w})$. Thus $M_k$ converges outer semicontinuously to $M$.\\

	Condition 1: $F$ is level-bounded, thus
	\begin{equation} \nonumber
		K(w^0) = \{w: F(w)\leq F(w^0)+\sum_k \epsilon_k \}
	\end{equation}
	is bounded. Since $F$ is also lsc, $K(w^0)$ is closed and hence compact.
	By (1) of Lemma \ref{thm:zangpr}, $w^k$ all lie in $K(w^0)$.\\

	Condition 2a: If $w^*\notin \Gamma$, then
	\begin{equation} \nonumber
		0\notin \partial F(w^*) \supset \partial U(w^*|w^*).
	\end{equation}
	By the generalized Fermat theorem,
	$w^*$ is not a minima of $U(\cdot|w^*)$, and hence $w^*\notin M(w^*)$.
	It follows that $\AA w\in M(w^*)$,
	\begin{equation} \nonumber
		F(w) \leq U(w|w^*) < U(w^*|w^*) = F(w^*).
	\end{equation}
	
	Condition 2b: Let $v\in M_k(w)$, and then
	\begin{equation} \nonumber
		U(v|w) \leq \min U(\cdot|w) + \epsilon_k.
	\end{equation}
	Thus,
	\begin{equation} \nonumber
		F(v) \leq U(v|w)\leq \min U(\cdot|w) + \epsilon_k \leq U(w|w) + \epsilon_k  = F(w) + \epsilon_k.
	\end{equation}
	Therefore, all the conditions of Lemma \ref{thm:zangpr} are satisfied
	and we have finished the proof.
\end{proof}
\bigskip

Just like the proof of Theorem 2, Theorem~3 can be easily proved by directly utilizing the results of the above Lemma \ref{thm:mmerr}. We omit the proof here.

\end{document}